\documentclass[10pt,twocolumn,twoside]{IEEEtran}
\usepackage{url}                
\usepackage{times}
\usepackage{epsfig}
\usepackage{graphicx}
\usepackage{amsmath}
\usepackage{amsthm}
\newtheorem{theorem}{Theorem}
\usepackage{dsfont}
\usepackage{amssymb}
\usepackage{subfigure}          
\usepackage{multirow}           
\usepackage{bm}
\usepackage{color}
\usepackage{soul}
\usepackage{stfloats}
\usepackage{lineno}
\usepackage{slashbox}

\usepackage{algorithmic}
\usepackage{algorithm}

\ifCLASSINFOpdf
\else
\fi
\begin{document}
%
\title{Learning a Robust Representation via a Deep Network on Symmetric Positive Definite Manifolds}
%
%
%

\author{Zhi Gao, Yuwei Wu, Xingyuan Bu, and Yunde Jia \IEEEmembership{Member, IEEE}
\thanks{The authors are with Beijing Laboratory of Intelligent Information Technology, School of Computer Science, Beijing Institute of Technology (BIT), Beijing, China. Email:\{gaozhi\underline{\hspace{0.5em}}2017,wuyuwei,buxingyuan,jiayunde\}@bit.edu.cn.}
\thanks{Corresponding author: Yuwei Wu.}

}

\maketitle

\begin{abstract}
Recent studies have shown that aggregating convolutional features of a pre-trained Convolutional Neural Network (CNN) can obtain impressive performance for a variety of visual tasks. The symmetric Positive Definite (SPD) matrix becomes a powerful tool due to its remarkable ability to learn an appropriate statistic representation to characterize the underlying structure of visual features. In this paper, we propose to aggregate deep convolutional features into an SPD matrix representation through the SPD generation and the SPD transformation under an end-to-end deep network. To this end, several new layers are introduced in our network, including a nonlinear kernel aggregation layer, an SPD matrix transformation layer, and a vectorization layer. The nonlinear kernel aggregation layer is employed to aggregate the convolutional features into a real SPD matrix directly. The SPD matrix transformation layer is designed to construct a more compact and discriminative SPD representation. The vectorization and normalization operations are performed in the vectorization layer for reducing the redundancy and accelerating the convergence. The SPD matrix in our network can be considered as a mid-level representation bridging convolutional features and high-level semantic features. To demonstrate the effectiveness of our method, we conduct extensive experiments on visual classification. Experiment results show that our method notably outperforms state-of-the-art methods.

\end{abstract}


%
\IEEEpeerreviewmaketitle

\section{Introduction}
%
%
%
%
\IEEEPARstart{D}{eep} Convolutional Neural Networks (CNNs) have shown great success in many vision tasks. There are several successful networks, \emph{e.g.,} AlexNet \cite{Krizhevsky2012ImageNet}, VGG \cite{Simonyan2014Very}, GoogleNet \cite{Szegedy2014Going}, Network In Network \cite{Lin2013Network} and ResNet \cite{He2015Deep}. Driven by the emergence of large-scale data sets and fast development of computation power, features based on CNNs have proven to perform remarkably well on a wide range of visual recognition tasks \cite{Zeiler2013Visualizing,Donahue2013DeCAF}. Two contemporaneous works introduced by Liu \emph{et al.} \cite{Lin2016Bilinear} and Babenko and Lempitsky \cite{Yandex2016Aggregating} demonstrate that convolutional features could be seen as a set of local features which can capture the visual representation related to objects. To make better use of deep convolutional features, many efforts have been devoted to aggregating them, such as max pooling \cite{Tolias2015Particular}, cross-dimensional pooling \cite{Kalantidis2015Cross}, sum pooling \cite{Yandex2016Aggregating}, and bilinear pooling \cite{Lin2016Bilinear,Gao2015Compact}. However, modeling these convolutional features to boost the feature learning ability of a CNN is still a challenging task. This work investigates a more effective scheme to aggregate convolutional features to produce a robust representation using an end-to-end deep network for visual tasks.

Recently, the Symmetric Positive Definite (SPD) matrix has been demonstrated the powerful representation ability and widely used in computer vision community, such as face recognition \cite{Huang2016Geometry,Harandi2017Dimensionality}, image set classification \cite{Huang2015Log}, transfer learning \cite{Herath2017Learning}, and action recognition \cite{Zhang2016Exploiting,Zhou2017Revisiting}. Through the theory of non-Euclidean Riemannain geometry, the SPD matrix often turns out to be better suited in capturing desirable data distribution properties. Accordingly, we attempt to aggregate the deep convolutional features into an powerful SPD matrix as a robust representation.

The second-order statistic information of convolutional features, \emph{e.g.,} the covariance matrix and Gaussian distribution, are the widely used SPD matrix representation endowed with CNNs \cite{Li2017Is,Ionescu2015Matrix,Yu2017Second}. The dimensionality of convolutional features extracted from CNNs may be much larger than that of hand-craft features. As a result, modeling convolutional features from CNNs by using the covariance matrix or Gaussian distribution is insufficient to precisely model the real feature distribution. When the dimension of features is larger than the number of features, the covariance matrix and Gaussian distribution is a symmetric Positive SemiDefinite (PSD) matrix, \emph{i.e.,} the singular matrix. Singular matrix makes the data have an unreasonable manifold structure. In this case, the Riemannain metrics, \emph{e.g.,} the affine-invariant distance and Log-Euclidean distance, are unsuitable to measure the manifold structure of SPD matrices. Moreover, most SPD matrix embedding on deep networks only contains the linear correlation of features. Owning the ability of capturing nonlinear relationship among features is indispensable for a generic representation.

It is thus desirable that a more discriminative and suitable SPD representation aggregated from deep features should be established in an end-to-end framework for visual analysis. To this end, we design a series of new layers to overcome existing issues aforementioned based on the following two observations.

\begin{figure*}[!t]
\centering
\includegraphics[width=7.0in]{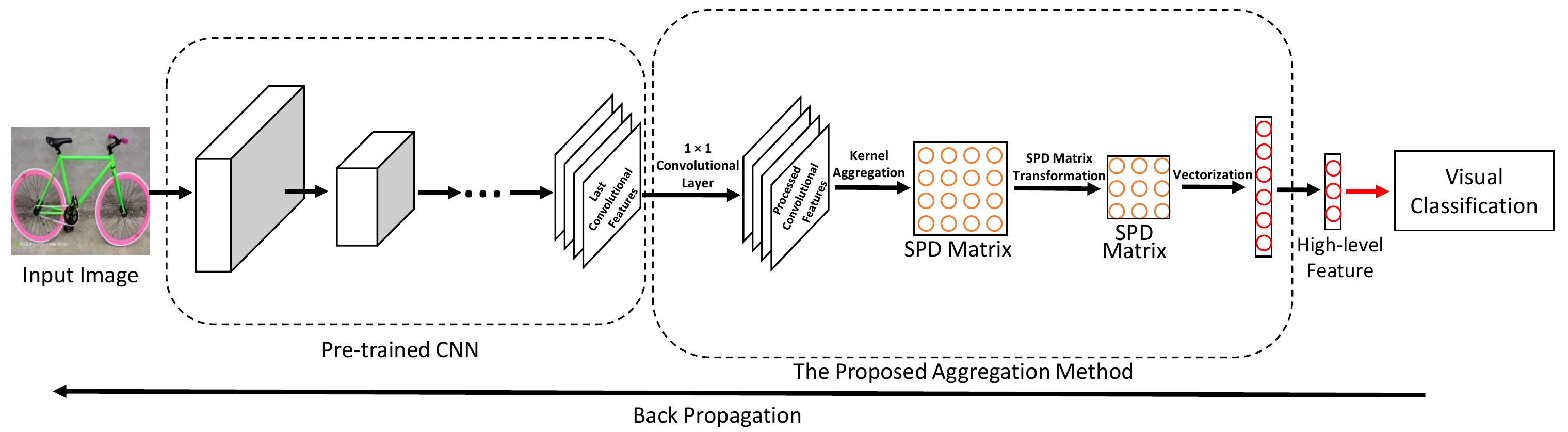}
\caption{The flowchart of our SPD aggregation network. We focus on both how to generate an SPD matrix from convolutional features and gain a more discriminative SPD representation by the transformation operation. Our general representation can be used for different tasks. To gain the processed convolutional features, the PCA and $1 \times 1$ convolutional layer are applied on the last convolutional features from the pre-trained CNN network. The kernel aggregation operation, SPD matrix transformation operation and vectorization operation correspond to the proposed three new layers. The kernel aggregation operation is to generate an SPD representation from the convolutional features. The SPD matrix transformation operation is to map the SPD matrix to a more compact and discriminative one by learnable parameters. The vectorization operation involves vectorization and normalization operations on the transformed SPD matrix.}
\label{fig:globalprocess}
\end{figure*}

\begin{itemize}
\item Kernel functions possess an ability of modeling nonlinear relationships of data, and they are fiexible and easy to be computed. Beyond covariance \cite{Wang2015Beyond} have witnessed significant advances of positive definite kernel functions whose kernel matrices are real SPD matrices, no matter what the number of the feature dimension and the number of features are. Since many kernel functions are differentiable, such as Radial Basis Function (RBF) kernel function, Polynomial kernel function and Laplacian kernel function \cite{Bo2010Kernel}, they can be readily embeded into a network to implement an end-to-end training, which is well aligned with the design requirements of a deep network.

\item Several deep SPD networks \cite{Dong2017Deep,Huang2016A,Zhang2017Deep} transform the SPD matrix to a new compact and discriminative matrix. The network input is an SPD matrix. The transformed matrix is still an SPD matrix which can capture desirable data properties. We find that the transformed SPD matrix after the learnable layers leads to better performance than the original SPD matrix. The output SPD matrix not only have characteristics of a general SPD matrix that captures the desirable properties of visual features but also is more suitable and discriminative to the specific visual task. 
\end{itemize}

Motivated by empirical observations mentioned above, we introduce a convolutional feature aggregation operation which consists of the SPD generation and the SPD transformation. Three new layers including a kernel aggregation layer, an SPD matrix transformation layer and a vectorization layer, are designed to replace the traditional global pooling layers and fully connected (FC) layers. Concretely, we deem each feature map as a sample and present a kernel aggregation layer using a nonlinear kernel function to generate an SPD matrix. The proposed kernel matrix models a nonlinear relationship among feature maps and ensures that the SPD matrix is nonsingular. More importantly, our kernel matrix is differentiable, which entirely meets requirements of a deep network. The SPD matrix transformation layer is employed to map the SPD matrix to a more discriminative and compact one. Thanks to the symmetry property, the vectorization layer carries out the upper triangle vectorization and normalization operations to the SPD matrix followed by the classifier. The architecure of our network is illustrated in Fig.~\ref{fig:globalprocess}. The proposed method first generates an SPD matrix based on convolutional features and then transforms the initial SPD matrix to a more discriminative one. It can not only capture the real spatial information but also encode high-level variation information among convolutional features. Actually, the obtained descriptor acts as a mid-level representation bridging convolutional features and high-level semantics features. The resulting vector can contribute to visual classification tasks, as validated in experiments.

In summary, our contributions are three-fold.

(1) We apply the SPD matrix non-linear aggregation to the convolutional feature aggregation field by the generation and the transformation two processes. In this way, it can learn an compactness and robustness SPD matrix representation to characterize the underlying structure of convolutional features.

(2) We carry out the nonlinear aggregation of convolutional features under a Riemannain deep network architecture, where three novel layers are introduced, \emph{i.e.,} a kernel aggregation layer, an SPD matrix transformation layer and a vectorization layer. The state-of-the-art performance of our SPD aggregation network is consistently achieved over the visual classification tasks.

(3) We exploit the faster matrix operations to avoid the cyclic calculation in forward and backward backpropagations of the kernel aggregation layer. In addition, we present the component decomposition and retraction of the Orthogonal Stiefle manifold to carry out the backpropagation on the SPD matrix transformation layer. 

The remaining sections are organized as follows. We review the recent works about feature aggregation methods in both Euclidean Space and Riemannain Space in Section \ref{relatedwork}. Section \ref{SPDAggregationMethod} presents the details of our SPD aggregation method.  We report and discuss the experimental results in Section \ref{experiment}, and conclude the paper in Section \ref{conclusion}.

\section{Related Work}

\label{relatedwork}

Feature aggregation is an important problem in computer vision tasks. Recent works have witnessed significant advances of CNNs. It is still a challenging work to find a suitable way to aggregate convolutional features. In this section, we review typical techniques of feature aggregation in both the Euclidean space and Riemannain space.

\subsection{Convolutional Feature Aggregation in the Euclidean Space}

An effective image representation is an essential element for visual recognition due to the object appearance variations caused by pose, view angle, and illumination changes. Traditional methods typically obtain the image representation by aggregating hand-crafted local features (\emph{e.g.,} SIFT) into a global image descriptor. Popular aggregation schemes include Bag-of-words (BOW) \cite{Sivic2003Video}, Fisher Vector (FV) \cite{Liu2014Encoding}, and Vector of Locally Aggregated Descriptor (VLAD) \cite{Ng2015Exploiting}. Gong \emph{et al.} \cite{Gong2014Multi} introduced a multi-scale orderless pooling scheme to aggregates FC6 features of local patches into a global feature using VLAD. The VLAD ignores different effects of each cluster center. Cimpoi \emph{et al.} \cite{Cimpoi2015Deep} treated the convolutional layer of CNNs as a filter bank and built an orderless representation using FV. In addition, Liu \emph{et al.} \cite{Liu2015Cross} proposed a cross convolutional layer pooling scheme which regards feature maps as a weighting filter to the local features. Tolias \emph{et al.} \cite{Tolias2015Particular} max pooled convolutional features of the last convolutional layer to represent each patch and achieved compelling performance for object retrieval. Babenko \emph{et al.} \cite{Yandex2016Aggregating} compared different kinds of aggregation methods (\emph{i.e.,} max pooling, sum pooling and fisher vector) for last convolutional layer features and demonstrated the sum-pooled convolutional descriptor is really competitive with other aggregation schemes.

Works mentioned above only treat the CNN as a black-box feature extractor rather than studying on properties of CNN features in an end-to-end framework. Several researchers \cite{Lin2016Bilinear,Zhang2016Deep,Arandjelovic2016NetVLAD} suggested that the end-to-end network can achieve better performance because it is sufficient by itself to discover good features for visual tasks. Arandjelovic \emph{et al.} \cite{Arandjelovic2016NetVLAD} proposed a NetVLAD which adopts an the end-to-end framework for weakly supervised place recognition. Based on the ResNet, Zhang \emph{et al.} \cite{Zhang2016Deep} introduced an extended version of the VLAD, \emph{i.e.,} Deep-TEN, for texture classification. Lin \emph{et al.} \cite{Lin2016Bilinear} presented a general orderless pooling model named Bilinear to compute the outer product of local features. He \emph{et al.} \cite{He2015Spatial} introduced a spatial pyramid pooling method eliminating the constrain of the fixed-size input image.

Recent research shows that exploiting the manifold structure representation is more effective than the hypothetical Euclidean distribution in some visual tasks. The difference between our method and the traditional aggregation methods in the Euclidean space is that we use the powerful SPD manifold structure to aggregate the desirable data distributions of features. We design an SPD aggregation scheme to generate the SPD matrix as the resulting representation, and transform the SPD representation to more discriminative one by learnable layers.

\subsection{Convolutional Feature Aggregation in the Riemannain Space}

The aggregation methods in the non-Euclidean space have been successful applied. It can capture more appropriate feature distributions information. The second-order statistic information has better performance than the first-order statistic \cite{Li2017Is}, such as average pooling. Some works directly regard the second-order statistic information as the SPD matrix. Ionescu \emph{et al.} \cite{Ionescu2015Matrix} proposed a DeepO2P network that uses a covariance matrix as the image representation. They mapped points on the manifold to the logarithm tangent space and derived a new chain rule for derivatives. Li \emph{et al.} \cite{Li2017Is} presented a matrix normalized covariance method exploring the second-order statistic. This work can tackle the singular issue of the covariance matrix by the normalization operation. Yu and Salzmann \cite{Yu2017Second} introduced a covariance descriptor unit to integrates second-order statistic information. The covariance matrix of convolutional features is generated and then transformed to a vector for the softmax classifier. Compared with our network, these three works are confined to the drawbacks of covariance matrices. Engin \emph{et al.} \cite{Engin2017DeepKSPD} designed a deep kernel matrix based SPD representation, but didn't contains the transformation process.

Other SPD Riemannain networks mainly project an SPD matrix to a more discriminative one. Dong \emph{et al.} \cite{Dong2017Deep} and Huang and Gool \cite{Huang2016A} proposed Riemannain networks contemporaneously, in which the inputs of their networks are SPD matrices. The networks projects high dimensional SPD matrices to a low dimensional discriminative SPD manifold by a nonlinear mapping. Zhang \emph{et al.} \cite{Zhang2017Deep} introduced new layers to transform and vectorize the SPD matrix for action recognition, where the input is a nonlinear kernel matrix modeling correlation of frames in a video. However, these three works only focused on how to transform the SPD matrix without utilizing the powerful convolutional features. The generation of the input SPD matrix can not be guided by the loss function. In contrast, our method focuses on not only the SPD matrix transformation but also the generation from convolutional features.

Our work is closely related with \cite{Li2017Is,Yu2017Second,Engin2017DeepKSPD}. We make it clear that the proposed convolutional feature aggregation method is composed of generation and transformation processes. Compared with \cite{Li2017Is}, our method utilizes the kernel matrix as the representation instead of the second-order statistic covariance matrix, characterizing complex nonlinear variation information of features. In addition, our aggregation method contains a learnable transformation process than \cite{Li2017Is}, making SPD representation more compact and robust. The generated SPD matrix in our method is more powerful than the covariance matrix in \cite{Yu2017Second}, avoiding some drawbacks of PSD matrix. In addition, instead of a transformation from a matrix to a vector, the vectorization operation in our work is taking the upper triangle of a matrix since there are already transformation operations between the SPD matrices. Compared to \cite{Engin2017DeepKSPD}, our SPD representation can be more compact and robust through the transformation process.

\begin{figure}[!t]
\centering
\includegraphics[width=3.5in]{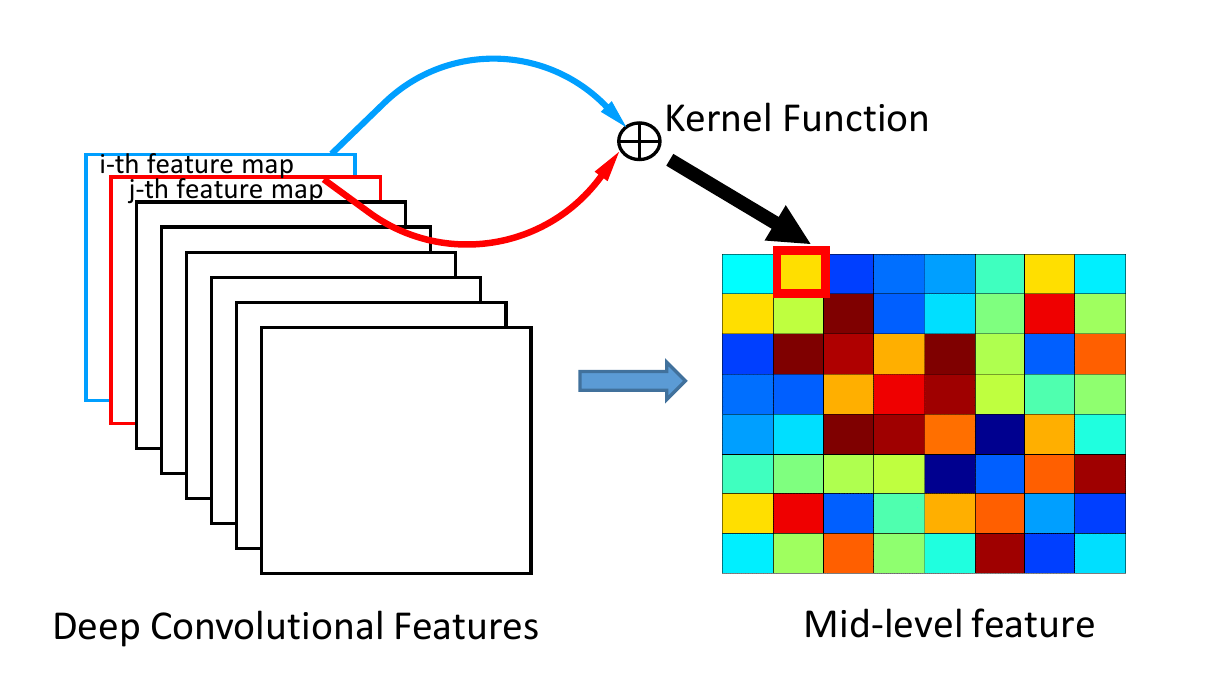}
\caption{Using the kernel function to generate an SPD matrix from convolutional features. The input is convolutional features and the output is an SPD kernel matrix.}
\label{fig:kernelaggregation}
\end{figure}

\section{SPD Aggregation Method}
\label{SPDAggregationMethod}

Our model aims to aggregate convolutional features into a powerful SPD representation in an end-to-end fashion. To this end, we design three novel layers including a kernel aggregation layer, an SPD matrix transformation layer and a vectorization layer. Our SPD aggregation can be applied to the visual classification. Specifically, the convolutional features pass through the proposed three layers followed by an FC layer and a loss function. The intermediate generated SPD matrix can be treated as a mid-level representation which is a connection between convolutional features and high-level features. The architecture of our network is illustrated in Fig.~\ref{fig:globalprocess}(c).

\subsection{Preprocessing of Convolutional Features}

A CNN model trained on a large dataset such as ImageNet can have a better general representation ability. We would like to fuse the convolutional features of the last convolutional layer and adjust the dimension of convolutional features for different tasks. We introduce a convolutional layer whose filter's size is $1 \times 1$ between the last convolutional layer of the off-the-shelf model and the kernel aggregation layer to make the processed convolutional features more adaptive to the SPD matrix representation. A Relu layer follows the $1 \times 1 $ convolutional layer to increase the nonlinear ability.

\subsection{Kernel Aggregation Layer}

We present the kernel aggregation layer to aggregate convolutional features into an initial SPD matrix. Let $X \in \mathbb{R}^ {C\times H \times W}$ be $3$-dimensional convolutional features. $C$ is the number of channels, \emph{i.e.,} the number of feature maps, $H$ and $W$ are the height and width of each feature map, respectively. Let $x_i \in \mathbb{R}^{C}$ denote the $i$-th local feature, and there are $N$ local features in total, where $ N = H \times W$. $f_i \in \mathbb{R}^{H \times W}$ is the $i$-th feature map.

Although several approaches have applied a covariance matrix $Cov$ to be a generic feature representation and obtained promising results, two issues remain to be addressed. First, the rank of covariance matrix should hold $rank(Cov) \leq \min(C, N-1)$, otherwise covariance matrix is prone to be singular when the dimension $C$ of local features is larger than the number of local features extracted from an image region. Second, for a generic representation, the capability of modeling nonlinear feature relationship is essential. However, covariance matrix only evaluates the linear correlation of features.

To address these issues, we adopt the nonlinear kernel matrix as a generic feature representation to aggregate deep convolutional features. In particular, we take advantage of the Riemannain structure of SPD matrices to describe the second-order statistic and nonlinear correlations among deep convolutional features. The nonlinear kernel matrix is capable of modeling nonlinear feature relationship and is guaranteed to be nonsingular. Different from the traditional kernel-based methods whose entries evaluates the similarity between a pair of samples, we apply the kernel mapping to each feature $f_1, f_2, \cdots, f_C$ rather than each sample $x_1, x_2, \cdots, x_N$. Mercer kernels are usually employed to carry out the mapping implicitly. The Mercer kernel is a function $\mathcal{K}(\cdot,\cdot)$ which can generate a kernel matrix $K \in \mathbb{R}^{C \times C}$ using pairwise inner products between mapped convolutional features for all the input data points. The $K_{ij}$ in our nonlinear kernel matrix $K$ can be defined as

\begin{equation}
\label{equation:forwardkernelaggregation}
K_{ij}= \mathcal{K} \left(f_i,f_j\right) = \big\langle \phi(f_i), \phi(f_j) \big\rangle,
\end{equation}
where $\phi (\cdot)$ is an implicit mapping. In this paper, we exploit the Radial Basis Function (RBF) kernel function expressed as

\begin{equation}
\label{equation:rbfkernelfunction}
\mathcal{K}\left( f_i,f_j \right) = exp \left(-\lVert f_i - f_j \lVert ^ {2} / 2\sigma^{2} \right),
\end{equation}
where $\sigma$ is a positive constant and set to the mean Euclidean distances of all feature maps. What Eq.~\eqref{equation:rbfkernelfunction} reveals is the nonlinear relationship between convolutional features.

We show an important theorem for kernel aggregation operation. Based on the Theorem~\ref{theorem:MicchelliTheorem}, the kernel matrix $K$ of the RBF kernel function is guaranteed to be positive definite no matter what $C$ and $N$ are.

\begin{theorem}
\label{theorem:MicchelliTheorem}
Let $\mathbf{X}=\{ \mathbf{x}_i \}_{i=1}^{M}$ denotes a set of different points and $\mathbf{x}_i \in \mathbb{R}^{n} $. Then the kernel matrix $K \in \mathbb{R}^{M \times M}$ of the RBF kernel function $\mathcal{K}$ on $\mathbf{X}$ is guaranteed to be a positive definite matrix, whose $(j,k)$-th element is ${K}_{jk} = \mathcal{K} \left(\mathbf{x}_j,\mathbf{x}_k\right)= exp \left(- \alpha \lVert \mathbf{x}_j - \mathbf{x}_k \lVert ^ {2} \right)$ and $\alpha > 0$.
\end{theorem}

\begin{proof}
The Fourier transform convention $\hat{\mathcal{K}}(\xi)$ of the RBF kernel function $\mathcal{K} \left(\mathbf{x}_i,\mathbf{x}_j\right) = exp \left(- \alpha \lVert \mathbf{x}_j - \mathbf{x}_k \lVert ^ {2} \right)$ is 
\begin{equation}
\label{equation:Fouriertransform}
\hat{\mathcal{K}}(\xi)= {(2\pi / \alpha)}^{n/2} \int_{R^n} e^{i\xi \mathbf{x}_{j}} e^{-i\xi \mathbf{x}_{k}} e^{-\lVert \xi \lVert ^{2} / {2\alpha}} d\xi.
\end{equation}
Then we calculate the quadratic form of the kernel matrix $K$. Let $\bm{c}=(c_1,\cdots,c_M) \in \mathbb{R}^{M \times 1}$ denote an arbitrary nonzero vector. The quadratic form $Q$ is
\begin{equation}
\begin{aligned}
& Q=\bm{c}^{\top}K\bm{c} = \sum_{j=1}^{M} \sum_{k=1}^{M} c_j c_k exp \left(- \alpha \lVert \mathbf{x}_j - \mathbf{x}_k \lVert ^ {2} \right) \\
& = \sum_{j=1}^{M} \sum_{k=1}^{M} c_j c_k {(2\pi / \alpha)}^{n/2} \int_{R^n} e^{i\xi \mathbf{x}_{j}} e^{-i\xi \mathbf{x}_{k}} e^{-\lVert \xi \lVert ^{2} / {(2\alpha)}} d\xi \\
& = {(2\pi / \alpha)}^{n/2} \int_{R^n} e^{-\lVert \xi \lVert ^{2} / {(2\alpha)}} \lVert \sum_{k=1}^{M} c_k e^{-i\xi \mathbf{x}_{k}} \lVert^{2} d\xi ,
\end{aligned}
\end{equation}
where $\top$ is the transpose operation. Because $e^{-\lVert \xi \lVert ^{2} / {(2\alpha)}}$ is a positive and continuous function, the quadratic form $Q = 0$ on the condition that 
\begin{equation}
\sum_{k=1}^{M} c_k e^{-i\xi x_{k}} = 0.
\end{equation}
However, the complex exponentials $e^{-i\xi \mathbf{x}_1},\cdots,e^{-i\xi \mathbf{x}_M}$ is linear independence. Accordingly, $Q > 0$ and kernel matrix $K$ is a positive definite matrix.
\end{proof}

In this work, $K$ is the generated SPD matrix as the mid-level image representation. Any SPD manifold optimization can be applied directly, without structure being destroyed. The toy example of the kernel aggregation is illustrated in Fig.~\ref{fig:kernelaggregation}. As we all known, the kernel aggregation layer should be differentiable to meet the requirement of an end-to-end deep learning framework. Clearly, Eq.~\eqref{equation:rbfkernelfunction} is differentiable with respect to the input $X$. Denoting by $\mathcal{L}$ the loss function, the gradient with respect to the kernel matrix is $\frac{\partial{\mathcal{L}}}{\partial{K}}$. $\frac{\partial{\mathcal{L}}}{\partial{K_{ij}}}$ is an element in $\frac{\partial{\mathcal{L}}}{\partial{K}}$. We compute the partial derivatives of $\mathcal{L}$ with respect to $f_i$ and $f_j$, which are

\begin{equation}
\label{equation:gradienttofifj}
\begin{aligned}
& \frac{\partial{\mathcal{L}}}{\partial{f_i}} = \sum_{j=1}^{H \times W} \frac{\partial{\mathcal{L}}}{\partial{K_{ij}}}\frac{K_{ij}}{-2\sigma^2}\left( f_i - f_j\right) \\
& \frac{\partial{\mathcal{L}}}{\partial{f_j}} = \sum_{i=1}^{H \times W} \frac{\partial{\mathcal{L}}}{\partial{K_{ij}}}\frac{K_{ij}}{-2\sigma^2}\left( f_j - f_i\right).
\end{aligned}
\end{equation}
In this process, the gradient of the SPD matrix can flow back to convolutional features.

During forward propagation Eq.~\eqref{equation:rbfkernelfunction} and backward propagation Eq.~\eqref{equation:gradienttofifj}, we have to do $C^{2}$ cycles to compute the kernel matrix $K$ and $2C^{2}$ cycles to gain the gradient with respect to convolutional features $\frac{\partial{\mathcal{L}}}{\partial{X}}$, where $C$ is the number of channels. Obviously, both the forward and backward propagations are computationally demanding. It is well known that the computation using matrix operations is preferable due to the parallel computing in computers. Accordingly, our kernel aggregation layer is able to be calculated in a faster way via matrix operations. Let's reshape the convolutional features $X \in \mathbb{R}^ {C\times H \times W}$ to a matrix $M \in \mathbb{R}^{C \times N}$. Each row of $M$ is a reshaped feature map $\mathbf{f}_{i} \in \mathbb{R}^{1 \times N}$ obtained from $f_i$ and each column of $M$ is the convolutional local feature $x_i$. Note that, $\lVert f_i - f_j \lVert ^ {2}$ in Eq.~\eqref{equation:rbfkernelfunction} can be expanded to $\lVert f_i - f_j \lVert ^ {2} = \mathbf{f}_{i}\mathbf{f}_{i}^{\top} - {2} \mathbf{f}_{i}\mathbf{f}_{j}^{\top} + \mathbf{f}_{j}\mathbf{f}_{j}^{\top}$. For each of inner products $\mathbf{f}_{i}\mathbf{f}_{i}^{\top}$, ${2} \mathbf{f}_{i}\mathbf{f}_{j}^{\top}$ and $\mathbf{f}_{j}\mathbf{f}_{j}^{\top}$, it needs to be calculated $C^{2}$ times in cycles of Eq.~\eqref{equation:rbfkernelfunction}. Now, we can convert $C^{2}$ times inner products operation to a matrix multiplication operation which only needs to be computed once, 
\begin{equation}
\label{equation:K1K2K3}
\begin{aligned}
& K1 = \mathbf{1} \left( M {\circ} M \right)^{\top} \\
& K2 = \left( M {\circ} M\right) \mathbf{1}^{\top} \\
& K3 = MM^\top, 
\end{aligned}
\end{equation}
where ${\circ}$ is the Hadamard product and $\mathbf{1} \in \mathbb{R}^{C \times N}$ is a matrix whose elements are all ``1"s. $K1$, $K2$ and $K3$ are all $C \times C$ real matirces. The element $K1\left(i , j\right)$ is the 2-norm of $i$-th row vector of $M$, and is equal to the calculation output of $\mathbf{f}_{i}\mathbf{f}_{i}^{\top}$. The element $K2\left(i , j\right)$ is the 2-norm of $j$-th column vector of $M$, and is equal to the calculation output of $\mathbf{f}_{j}\mathbf{f}_{j}^{\top}$. The element $K3\left(i , j\right)$ is equal to $\mathbf{f}_{i}\mathbf{f}_{j}^{\top}$. $K1$, $K2$ and $K3$ can be calculated in advance.

Therefore, we compute $-\lVert f_i - f_j \lVert ^ {2} / 2\sigma^{2}$ in Eq.~\eqref{equation:rbfkernelfunction} by the matrix addition and multiplication, and implement the $exp\left( \cdot \right)$ to the matrix in a parallel computing way instead of calculating each element in the cycle. Then the kernel matrix $K$ can be calculated by matrix operations as follows.

\begin{equation}
\label{equation:forwardmatrix}
K = exp\left(-\left(K1 + K2 - {2}K3\right) /2\sigma^{2}\right),
\end{equation}
where $exp\left( A \right)$ means the exponential operation to each element in the matrix $A$. Although calculating directly the $exp\left( \cdot \right)$ function is time-consuming, it can be computed efficiently in a matrix form through Eq.~\eqref{equation:forwardmatrix}, which is faster than through Eq.~\eqref{equation:rbfkernelfunction}. Similarly, back propagation process in Eq.~\eqref{equation:gradienttofifj} can also be carried out in the matrix operation which is given by
\begin{equation}
\label{equation:backwardmatrix}
\frac{\partial{\mathcal{L}}}{\partial{M}} = 4 \left( \mathbf{1}^{\top} \left( \frac{ \frac{\partial{\mathcal{L}}}{\partial{K}}\circ K}{-2 \sigma^{2}} \right) \circ M^{\top} - M^{\top} \left( \frac{\frac{\partial{\mathcal{L}}}{\partial{K}}\circ K}{-2 \sigma^{2}} \right) \right) ^{\top}.
\end{equation}

\textbf{Remark:} The covariance matrix descriptor, as a special case of SPD matrices, captures feature correlations compactly in an object region, and therefore has been proven to be effective for many applications. Given the local features $x_1, x_2, \cdots, x_N$, the covariance descriptor $Cov$ is defined as
\begin{equation} \label{covariancematrix}
  Cov = \frac{1}{N-1} \sum_{i=1}^{N} ({x}_{i} - {\mu})({x}_{i} - {\mu})^{\top},
\end{equation}
where ${\mu} = \frac{1}{N}\sum_{i=1}^{N} {x}_{i}$ is the mean vector. The covariance matrix can also be seen as a kernel matrix where the $\left( i,j\right)$-th element of the covariance matrix $Cov$ can be expressed as

\begin{equation}
\label{equation:covarianceinnerproduct}
Cov_{ij} =   \Bigg\langle \frac{\overline{f_i}}{\sqrt{N-1}}, \frac{\overline{f_j}}{\sqrt{N-1}} \Bigg\rangle,
\end{equation}
where $\langle \cdot , \cdot \rangle$ denotes the inner product, $ \overline {f_i}  = f_i-\mu_i \mathbf{1}$ and $\mu_i$ is the mean value of $f_i$. Therefore, the covariance matrix corresponds to a special case of the nonlinear kernel matrix defined in Eq.~\eqref{equation:forwardkernelaggregation}, where $\phi(f_i) = (\overline{f_i} - \mu_i\mathbf{1})/ \sqrt{N-1} $. Through this way, we can find that covariance matrices contain the simple linear correlation features. Whether the covariance matrix is a positive definite matrix depends on the $C$ and $N$, \emph{i.e.,} $rank(Cov) \leq \min(C, N-1)$.

\subsection{SPD Matrix Transformation Layer}

As discussed in \cite{Dong2017Deep,Huang2016A, Zhang2017Deep}, SPD matrix transformation networks are capable of achieving the better performance than the original SPD matrix. Inspired by \cite{Yu2016Weakly} and \cite{Yu2017Second}, we add a learnable layer to make the network more flexible and more adaptive to the specific task. Based on the SPD matrix generated by the kernel aggregation layer, we expect to transform the existing SPD representation to be a more discriminative, suitable and desirable matrix. To preserve the powerful ability of the SPD matrix, the transformed matrix should also be an SPD matrix. Moreover, we attempt to adjust the dimension to make the SPD matrix more flexible and compact. Here, we design the SPD matrix transformation layer in our network.

\begin{figure}[!t]
\centering
\includegraphics[width=3.5in]{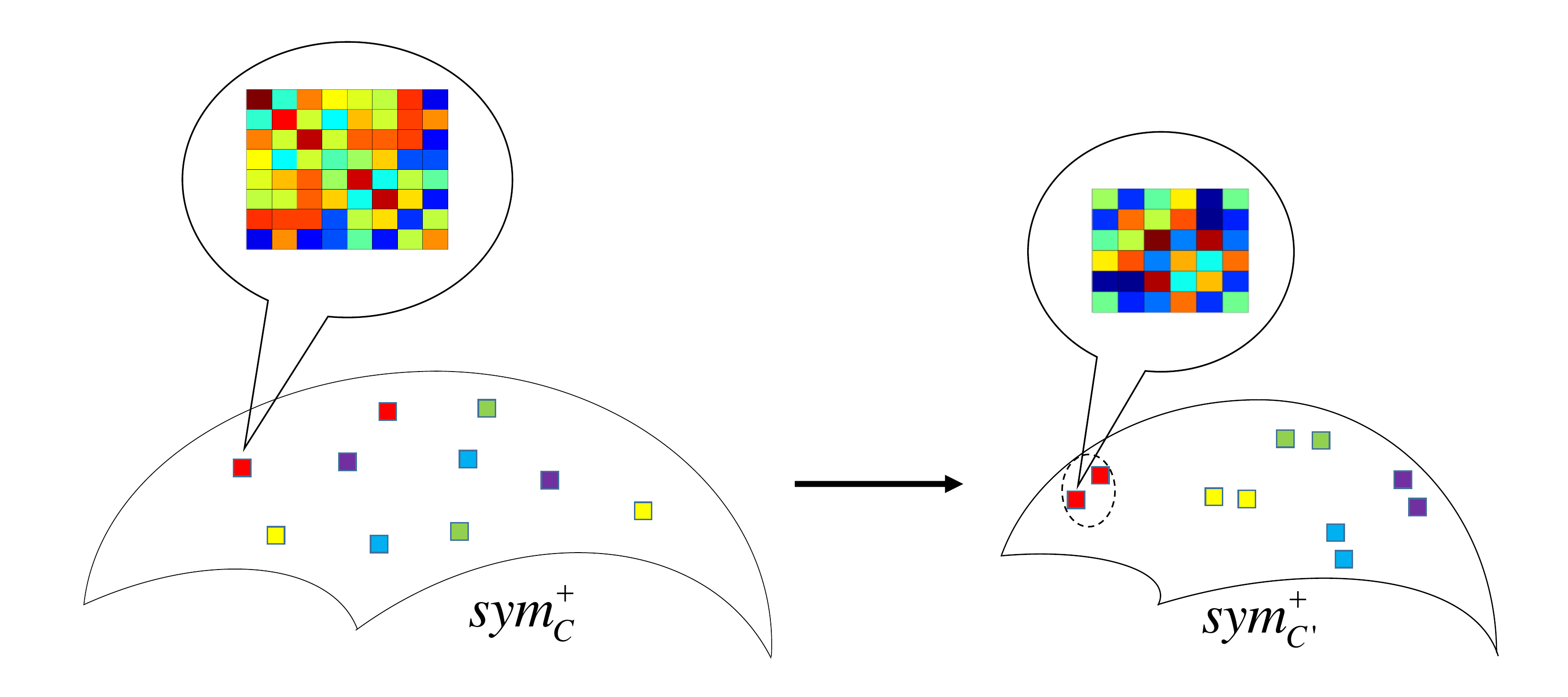}
\caption{The illustration of the projection from a manifold to another.}
\label{fig:spdmanifold}
\end{figure}

Let's define the Riemannain manifold of $n \times n$ SPD matrices as $Sym_n^{+}$. The output SPD matrix $K$ of the kernel aggregation layer lies on the manifold $Sym_C^{+}$. We use a matrix mapping to complete the transformation operation. As depicted in Fig.~\ref{fig:spdmanifold}, we map the input SPD matrix which lies on the original manifold $Sym_C^{+}$ to a new discriminative and compact SPD matrix in another manifold $Sym_{C^{\prime}}^{+}$, where $C^{\prime}$ is the dimension of the SPD matrix transformation layer. In this way, the desired transformed matrix can be obtained by a learnable mapping. Given a $C \times C$ SPD matrix $K$ as an input, the output SPD matrix can be calculated as

\begin{equation}
\label{equation:forwardmaniflodmapping}
Y = W^{\top}KW,
\end{equation}
where $Y \in \mathbb{R}^{{C}^{\prime} \times {C}^{\prime}}$ is the output of the transformation layer, and $W \in \mathbb{R}^{C \times C^{\prime}}$ are learnable parameters which are randomly initialized during training. $C^{\prime}$ controls the size of $Y$. Based on the Theorem \ref{theorem:columnfullrank}, the learnable parameters $W$ should be a column full rank matrix to make $Y$ be an SPD matrix as well. 

\begin{theorem}
\label{theorem:columnfullrank}
Let $A \in \mathbb{R}^{C\times C}$ denote an SPD matrix, $W \in \mathbb{R} ^ {C \times C'}$ and $B=W^{\top}AW$, where $C \geq C'$. $B$ is an SPD matrix if and only if $W$ is a column full rank matrix, \emph{i.e.,} $rank(W) = C'$.
\end{theorem} 

\begin{proof}
If $A$ is an SPD matrix, $W$ is a column full rank matrix and $rank(W) = C'$. For homogeneous equations $W\mathbf{x}=\mathbf{0}$ and $\mathbf{x} \in \mathbb{R}^{C' \times 1}$, $W\mathbf{x}=\mathbf{0}$ only has a zero solution, where $\mathbf{0}$ is the zero vector. For arbitrary nonzero vector $\mathbf{x}$, $W\mathbf{x} \not= \mathbf{0}$. We calculate the quadric form $\mathbf{x}^{\top}B\mathbf{x}$,
\begin{equation}
 \mathbf{x}^{\top}B\mathbf{x} = \mathbf{x}^{\top} W^{\top}AW \mathbf{x} = (W\mathbf{x})^{\top}A(W\mathbf{x}).
\end{equation}
Because $W\mathbf{x} \not= \mathbf{0}$ and $A$ is an SPD matrix, $\mathbf{x}^{\top}B\mathbf{x} > \mathbf{0}$. This proves that $B$ is an SPD matrix.

On the other hand, if $B$ is an SPD matrix, for arbitrary nonzero vector $\mathbf{x} \in \mathbb{R}^{C' \times 1}$, $\mathbf{x}^{\top}B\mathbf{x} = (W\mathbf{x})^{\top}A(W\mathbf{x}) > \mathbf{0}$. Beccause $A$ is an SPD matrix, $W\mathbf{x} \not=\mathbf{0}$. Only if $\mathbf{x}=\mathbf{0}$ can lead to $W\mathbf{x} =\mathbf{0}$. Accordingly, $rank(W) = C'$ and $W$ is a column full rank matrix.
\end{proof}

Since there are learnable parameters in the SPD matrix transformation layer, we should not only compute the gradient of loss function $\mathcal{L}$ with respect to the input $K$, but also calculate the gradient with respect to parameters $W$. The gradient with respect to the input $K$ is
\begin{equation}
\label{equation:backwardmaniflodmappingk}
\frac{\partial{\mathcal{L}}}{\partial{K}} = W\frac{\partial{\mathcal{L}}}{\partial{Y}}W^{\top},
\end{equation}
where $\frac{\partial{\mathcal{L}}}{\partial{Y}}$ is the gradient with respect to the output $Y$.

Since $W$ is a column full rank matrix, it is on a non-compact Stiefel manifold \cite{Absil2009Optimization}. However, directly optimizing $W$ on the non-compact Stiefel manifold is infeasible. To overcome this issue, we relax $W$ to be semi-orthogonal, \emph{i.e.,} $W^{\top}W=I_{C_{\prime}}$. In this case, $W$ is on the orthogonal Stiefel manifold $St \left(C^{\prime}, C \right)$. The optimization space of parameters $W$ is changed from the non-compact Stiefel manifold to the orthogonal Stiefel manifold $St \left(C^{\prime}, C \right)$. Considering the manifold structure of $W$, the optimization process is quite different from the gradient descent method in the Euclidean space. We first compute the partial derivative with respect to $W$. Then we convert the partial derivative to the manifold gradient that lies on the tangent space. Along the tangent gradient, we find a new point on the tangent space. Finally, the retraction operation is applied to map the new point on the tangent space back to the orthogonal Stiefel manifold. Thus, an iteration of the optimization process on the manifold is completed. This process is illustrate in Fig.~\ref{fig:optimizew}. Next we will elaborate each step.

First the partial derivative $\frac{\partial{\mathcal{L}}}{\partial{W}}$ with respect to $W$ is computed by
\begin{equation}
\label{equation:backwardmaniflodmappingweuclidean}
\frac{\partial{\mathcal{L}}}{\partial{W}} = K^{\top}W\frac{\partial{\mathcal{L}}}{\partial{Y}}+KW\left(\frac{\partial{\mathcal{L}}}{\partial{Y}}\right)^{\top}.
\end{equation}
The partial derivative $\frac{\partial{\mathcal{L}}}{\partial{W}}$ doesn't contain any manifold constraints. Considering $W$ is a point on the orthogonal Stiefel manifold, the partial derivative needs to be converted to the manifold gradient, which is on the tangent space. As shown in Fig.~\ref{fig:normaltangent}, on the orthogonal Stiefel manifold, the partial derivative $\frac{\partial{\mathcal{L}}}{\partial{W}}$ is a Euclidean gradient at the point $W$, not tangent to the manifold. The tangential component of $\frac{\partial{\mathcal{L}}}{\partial{W}}$ is what we need for optimization, which lies on the tangent space. The normal component is perpendicular to the tangent space. We decompose $\frac{\partial{\mathcal{L}}}{\partial{W}}$ into two vectors that are perpendicular to each other, \emph{i.e.,} one is tangent to the manifold and the other is the normal component based on the Theorem~\ref{theorem:OrthogonalStiefleManifold}.

\begin{theorem}
\label{theorem:OrthogonalStiefleManifold}
Let $M$ denote an orthogonal Stiefel manifold and $\mathbf{X}$ is a point on $M$. $F(\mathbf{X})$ denotes a function defined on the orthogonal Stiefel manifold. If the partial derivatives of $F$ with respect to $\mathbf{X}$ is $F_{\mathbf{X}}$, the manifold gradient $\nabla F$ at $\mathbf{X}$ which is tangent to $M$ is $\nabla F = F_{\mathbf{X}}- \mathbf{X} F_{\mathbf{X}}^{\top}\mathbf{X}$.
\end{theorem}

\begin{proof}
Because $\mathbf{X}$ is a point on the orthogonal Stiefel manifold, $\mathbf{X}^{\top}\mathbf{X}=I$, where $I$ is an identity matrix. Differentiating $\mathbf{X}^{\top}\mathbf{X}=I$ yields $\mathbf{X}^{\top} \Delta + {\Delta}^{\top}\mathbf{X} = 0$, where $\Delta$ is a tangent vector. Thus, $\mathbf{X}^{\top} \Delta $ is a skew-symmetric matrix. Note that, the canonical metric for the orthogonal Stiefel manifold at the point $Y$ is $g_c\left( {\Delta}_1, {\Delta}_2 \right) = tr \big({\Delta}_1^{\top} (I- \frac{1}{2} Y Y^{\top}) {\Delta}_2 \big)$. For all tangent vectors $\delta$ at $\mathbf{X}$, we can get that
\begin{equation}
\label{equation:tangent}
tr F_\mathbf{X}^{\top} \Delta = g_c(\nabla F, \Delta) = tr {(\nabla F)}^{\top} (I- \frac{1}{2} \mathbf{X} \mathbf{X}^{\top}) \Delta.
\end{equation}
Because $\mathbf{X}^{\top} (\nabla F)$ is a skew-symmetric matrix, Eq.~\eqref{equation:tangent} can be solved, \emph{i.e.,} $\nabla F = F_\mathbf{X }- \mathbf{X} F_\mathbf{X}^{\top}\mathbf{X}$.
\end{proof}

\begin{figure}[!t]
\centering
\includegraphics[width=3.5in]{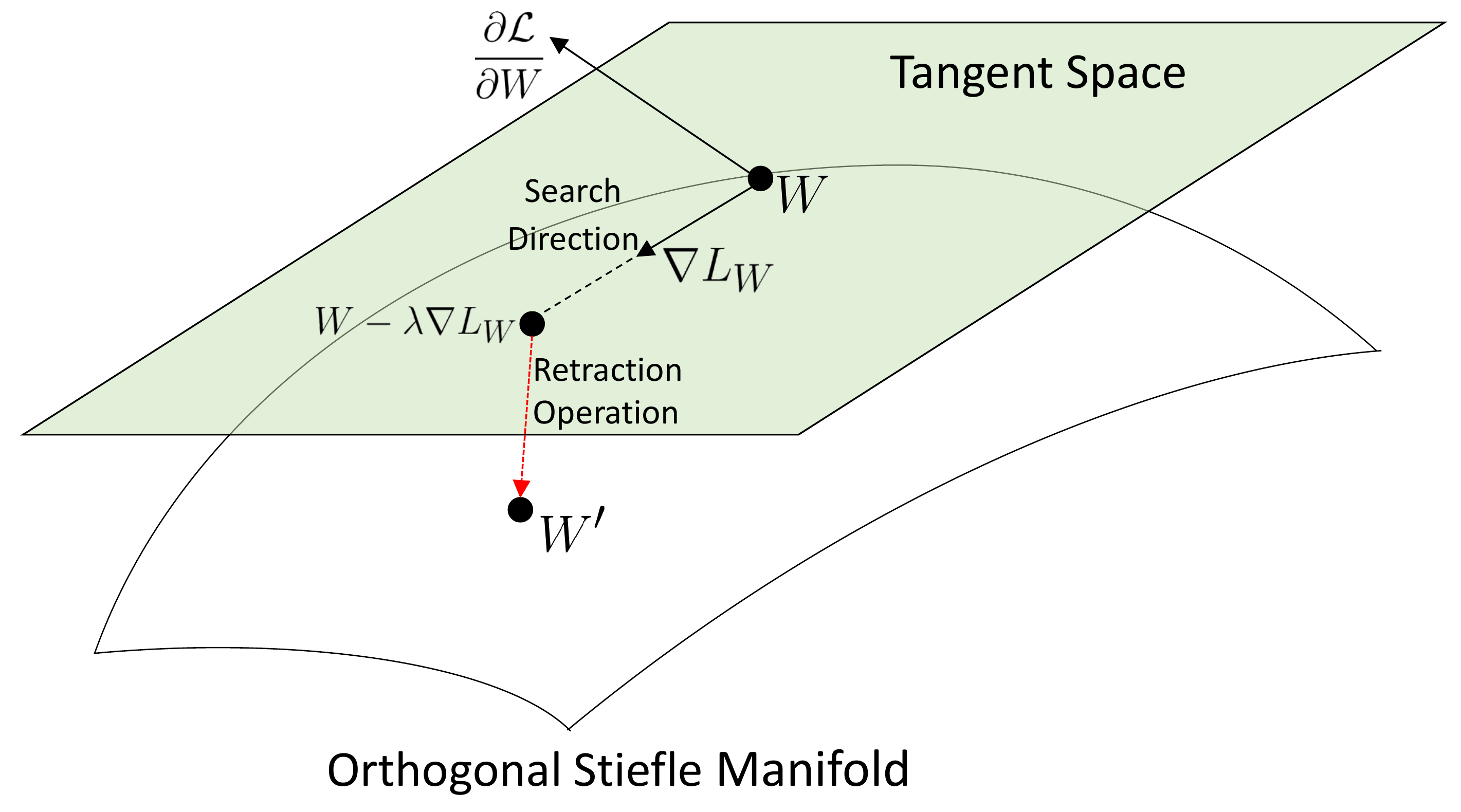}
\caption{The illustration of the optimization process of $W$. $W$ is an original point on the orthogonal Stiefel manifold. $W'$ is a new point after an iterative update. $\frac{\partial{\mathcal{L}}}{\partial{W}}$ is the partial derivative of the loss function with respect to $W$. $\nabla L_{W}$ is the manifold gradient lying on the tangent space.}
\label{fig:optimizew}
\end{figure}

Then the tangential component $\nabla L_{W}$ at $W$ can be expressed by the partial derivative $\frac{\partial{\mathcal{L}}}{\partial{W}}$,
\begin{equation}
\label{equation:backwardmaniflodmappingwmanifold}
\nabla L_{W} = \frac{\partial{\mathcal{L}}}{\partial{W}} - W \left({\frac{\partial{\mathcal{L}}}{\partial{W}}}\right)^{\top}W.
\end{equation}
$\nabla L_{W}$ is the manifold gradient of the orthogonal Stiefel manifold. Searching along the giadient $\nabla L_{W}$ gets a new point on the tangent space. Finally, we use the retracting operation to map the point on the tangent space back to the Stiefel manifold space,
\begin{equation}
\label{equation:backwardmaniflodmappingwretracting}
W := q\left(W - \lambda \nabla L_{W}\right),
\end{equation}
where $q\left( \cdot \right)$ is the retraction operation mapping the data back to the manifold. Specificly, $q\left( A \right)$ denotes the $Q$ matrix of QR decomposition to $A$. $A \in \mathbb{R}^{n \times p}$, $A = QR$, where $Q \in \mathbb{R}^{n \times p}$ is a semi-orthogonal matrix and $R \in \mathbb{R}^{p \times p}$ is a upper triangular matrix. $\lambda$ is the learning rate. 

Note that, we can make a Relu activation function layer follow the SPD matrix transformation layer. The output of the Relu layer is still an SPD matrix based on the Theorem~\ref{theorem:reluspdmatrix}.

\begin{figure}[!t]
\centering
\includegraphics[width=3.5in]{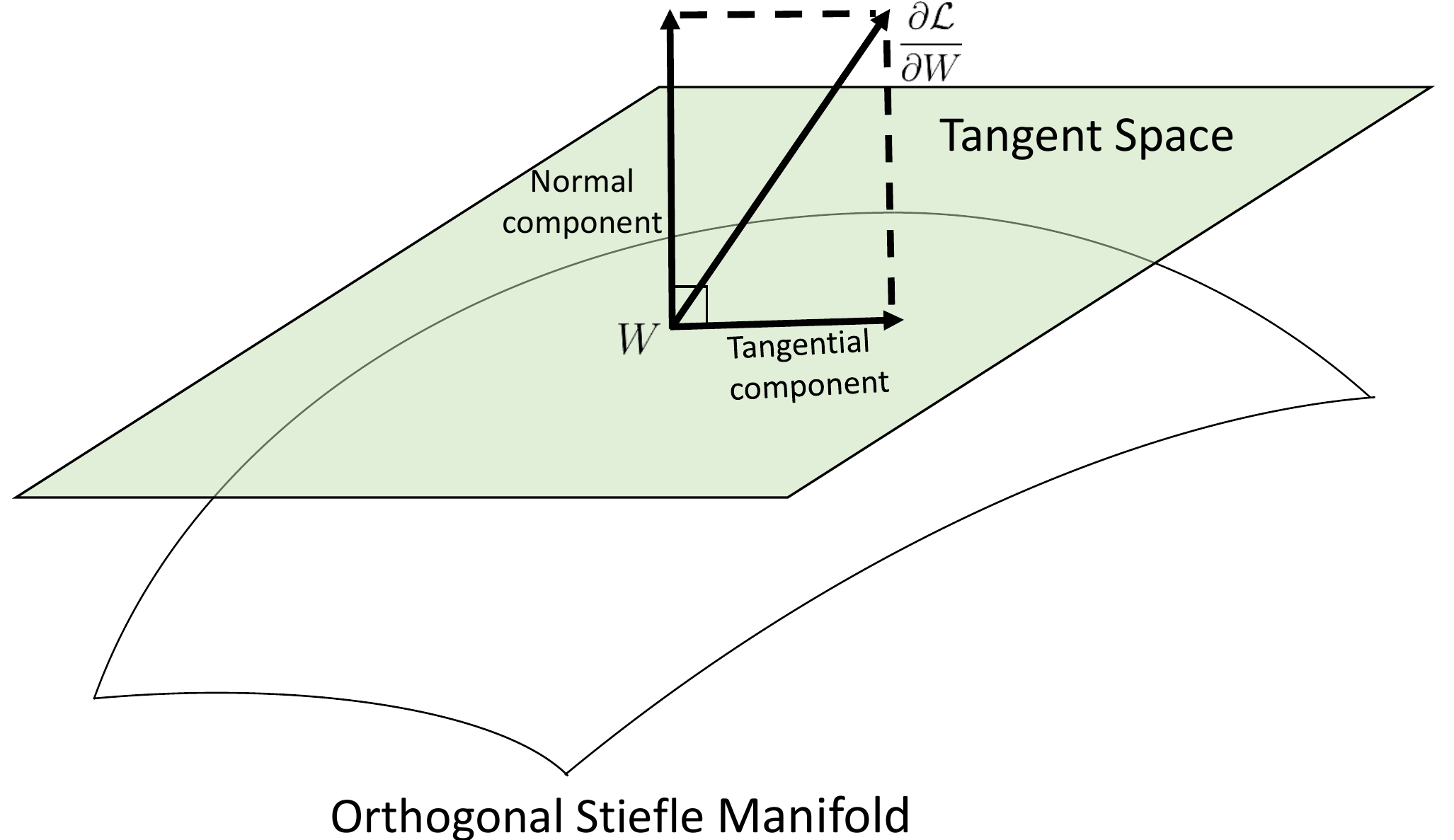}
\caption{The illustration of tangential and normal components of a vector. These two components are perpendicular. $v$ is the vector on the surface. The tangential component of $v$ is a vector on the tangent space and normal conpoment is perpendicular to the tangent space.}
\label{fig:normaltangent}
\end{figure}

\begin{theorem}
\label{theorem:reluspdmatrix}
The relu activation function on a matrix $Y$ is $f(Y)$. Let $Z=f(Y)$,
\begin{eqnarray} \nonumber
Z_{ij} = \left\{ \begin{array}{l}
0, \quad if \ \ Y_{ij}  \textless 0\\
Y_{ij},  \quad if \ \ Y_{ij} \geq 0
\end{array} \right..
\end{eqnarray}
If $Y$ is an SPD matrix, $Z$ is an SPD matrix.
\end{theorem}

\begin{proof}
The detailed proof of this theorem is shown in the appendix section of \cite{Dong2017Deep}.
\end{proof}

\subsection{Vectorization layer}
Since inputs of the common classifier is all vectors, we should vectorize the SPD matrix to a vector. Because of the symmetry of the robust SPD matrix achieved by the transformation layer, $Y$ is determined by $\frac{C^{\prime} \times \left( C^{\prime} + 1 \right)}{2}$ elements, \emph{i.e.,} the upper triangular matrix or the lower triangular matrix of $Y$. Here, we take the upper triangular matrix of $Y$ and reshape it into a vector as the input of the loss function. Let's denote the vector by $V$, 

\begin{equation}
\label{equation:vectorization}
\begin{split}
    \ & V = \Big [ Y_{11}, \sqrt{2}Y_{12}, ..., \sqrt{2}Y_{1{C^{\prime}}}, Y_{22}, \\
    & \sqrt{2}Y_{23}, ..., \sqrt{2}Y_{{C^{\prime}}{(C^{\prime}-1)}}, Y_{{C^{\prime}}{C^{\prime}}} \Big ] \\
    & = \Big [ V_1, V_2, ..., V_{\frac{C^{\prime} \times \left( C^{\prime} + 1 \right)}{2}} \Big ].
\end{split}
\end{equation}

Due to the symmetry of the matrix $Y$, the gradient $\frac{\partial{\mathcal{L}}}{\partial{Y}}$ is also a symmetric matrix. For the diagonal elements of $Y$, its gradient of the loss function is equal to the gradient of its corresponding element in the vector $V$, while the gradient of non-diagonal elements of $Y$ is $\sqrt{2}$ times of the element in the vector $V$. The gradient with respect to $Y$ is given by

\begin{small}
\begin{equation}
\label{equation:backvectorization}
\frac{\partial{\mathcal{L}}}{\partial{Y}} = \left[\begin{IEEEeqnarraybox*}[][c]{,c/c/c/c/c,}
\frac{\partial{\mathcal{L}}}{\partial{V_1}}&\frac{\sqrt{2}\partial{\mathcal{L}}}{\partial{V_2}}&\cdots&\frac{\sqrt{2}\partial{\mathcal{L}}}{\partial{V_{C^{\prime}-1}}}&\frac{\sqrt{2}\partial{\mathcal{L}}}{\partial{V_{C^{\prime}}}} \\
\frac{\sqrt{2}\partial{\mathcal{L}}}{\partial{V_2}}&\frac{\partial{\mathcal{L}}}{\partial{V_{C^{\prime}+1}}}&\frac{\sqrt{2}\partial{\mathcal{L}}}{\partial{V_{C^{\prime}+2}}}&\cdots&\frac{\sqrt{2}\partial{\mathcal{L}}}{\partial{V_{2 \times C^{\prime}-1}}} \\
\vdots&\vdots&\vdots&\cdots&\vdots \\
\frac{\sqrt{2}\partial{\mathcal{L}}}{\partial{V_{C^{\prime}}}}&\frac{\sqrt{2}\partial{\mathcal{L}}}{\partial{V_{2 \times C^{\prime}-1}}}&\frac{\sqrt{2}\partial{\mathcal{L}}}{\partial{V_{3 \times C^{\prime}-3}}}&\cdots&\frac{\partial{\mathcal{L}}}{\partial{V_{\frac{C^{\prime} \times \left( C^{\prime} + 1 \right)}{2}}}}
\end{IEEEeqnarraybox*}\right].
\end{equation}
\end{small}

\begin{figure}[!t]
\centering
\includegraphics[width=3.5in]{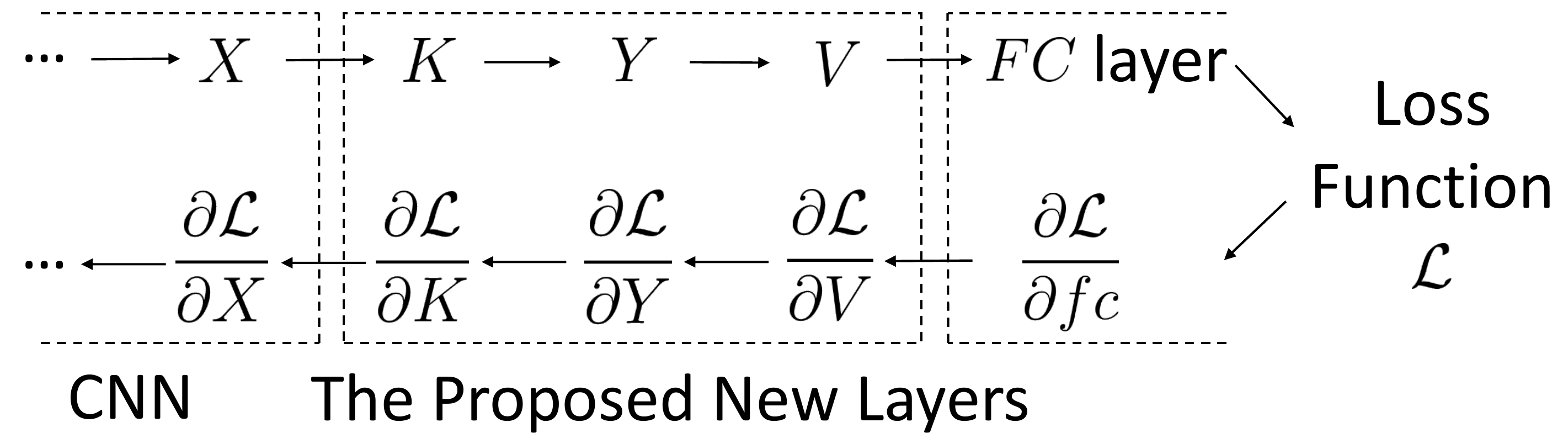}
\caption{Illustration of the forward and backward propagations of networks with the proposed aggregation method.}
\label{fig:forwardbackward}
\end{figure}

The normalization operation is important as well. We use the power normalization $\left( V_i:= sign \left(V_i \right) \sqrt{|V_i|} \right)$ and $l_2$ normalization $\left( V:= V/ {\|V\|}_2 \right)$ operation following the vector $V$. The gradient formulation Eq.~\eqref{equation:backwardmatrix}, Eq.~\eqref{equation:backwardmaniflodmappingk} and Eq.~\eqref{equation:backwardmaniflodmappingwretracting} calculate the gradient with respect to the input of the corresponding layer, respectively. Once these gradients are obtained, the standard SGD backpropagation can be easily employed to update the parameters directly with the learning rate. Fig.~\ref{fig:forwardbackward} shows the data flow in our network with the proposed three layers including forward and backward propagations. $fc$ denotes the output of the last fully-connected layer. In Algorithm~\ref{algorithm:UpdateTransformationlayer}, we summarize the training process of our model. We can use more than one SPD transformation layers in the network, where each one can be followed by a Relu layer as the activation layer.

\begin{algorithm}
 \caption{Trainging Process of Our Model}
 \begin{algorithmic}[1]
 \renewcommand{\algorithmicrequire}{\textbf{Input:}}
 \renewcommand{\algorithmicensure}{\textbf{Output:}}
 \REQUIRE Training data $I=\{I_i\}_{i=1}^n$, where $n$ is the number of training samples. SPD matrix transformation layer's initial parameters $W$. The other parameters in the CNN $\theta$. Learning rate $\lambda$. 
 \ENSURE  SPD matrix transformation layer's parameters $W$. Other layer's parameters $\theta$.

  \WHILE {not converge}
    \STATE Compute the convolutional features $M$ by forwarding $I_i$ through convolutional layer.
    \STATE Compute the SPD matrix $K$ by Eq.\eqref{equation:forwardmatrix}.
    \STATE Compute the transformed SPD matrix $Y$ by Eq.\eqref{equation:forwardmaniflodmapping}.
    \STATE Compute the vector representation $V$ by Eq.\eqref{equation:vectorization}.
    \STATE Compute the loss $\mathcal{L}$.
    \STATE Compute the gradient $\frac{\partial{\mathcal{L}}}{\partial{V}}$.
    \STATE Compute the gradient $\frac{\partial{\mathcal{L}}}{\partial{Y}}$ by Eq.\eqref{equation:backvectorization}.
    \STATE Compute the gradient $\frac{\partial{\mathcal{L}}}{\partial{K}}$ by Eq.\eqref{equation:backwardmaniflodmappingk}.
    \STATE Update the parameters $W$ by Eq.\eqref{equation:backwardmaniflodmappingweuclidean}, Eq.\eqref{equation:backwardmaniflodmappingwmanifold} and Eq.\eqref{equation:backwardmaniflodmappingwretracting}.
    \STATE Compute the gradient $\frac{\partial{\mathcal{L}}}{\partial{M}}$ with respect to the last convolutional features.
    \STATE Compute the gradient $\frac{\partial{\mathcal{L}}}{\partial{\theta}}$ with respect to the other parameters.
    \STATE Update the other parameters $\theta=\theta - \lambda {\frac{\partial{\mathcal{L}}}{\partial{\theta}}}$.
  \ENDWHILE
 \RETURN $\theta$ and $w$ 
 \end{algorithmic}
 \label{algorithm:UpdateTransformationlayer}
\end{algorithm}

\section{Experiment}
\label{experiment}

To demonstrate the benefits of our method, we conduct extensive experiments on visual classification tasks. We conduct experiments on visual classification tasks to show the performace of the SPD aggregation framework including the generation and transformation processes. We present the visual classification tasks on five datasets. We choose the challenging texture and fine-grained classification tasks. The texture classification tasks need a powerful global representation, because of the features of texture should be invariant to translation, scaling and rotation. Differences among fine-grained images are very small. It is challenging to represent these differences in the aggregation process.

\begin{table*}[!t]
\renewcommand{\arraystretch}{1.3}
\caption{Comparison for CNNs based methods in terms of Average Precision (\%). Our method is bold in the last line.}
\label{table:comparetextureandfinegrained}
\centering
\begin{tabular}{c|c c c c c}
\hline
\bfseries Method & \bfseries DTD & \bfseries FMD & \bfseries KTH-T2b & \bfseries CUB-200-2011 & \bfseries FGVC-aircraft \\
\hline
FV-CNN \cite{Cimpoi2015Deep} & $67.3$ & $73.5$ & $73.3$ & $49.9$ & -\\
\hline
FV-FC-CNN \cite{Cimpoi2015Deep} & $69.8$ & $76.4$ & $73.8$ & $54.9$ & - \\
\hline
B-CNN \cite{Lin2016Bilinear} & $69.$6 & $77.8$ & $79.7$ & $74.0$ & $74.3$\\
\hline
Deep-TEN$_\emph{ResNet50}$ \cite{Zhang2016Deep}& - & $80.2$ & $82.0$ & - & - \\
\hline
VGG-16 \cite{Simonyan2014Very}& $66.8$ & $77.8$ & $78.3$ & $68.0$ & $75.0$ \\
\hline
\textbf{Ours} & $\mathbf{68.9}$ & $\mathbf{79.2}$ & $\mathbf{81.1}$ & $\mathbf{72.4}$ & $\mathbf{77.8}$\\
\hline
\end{tabular}
\end{table*}

\begin{table*}[!t]
\renewcommand{\arraystretch}{1.3}
\caption{Comparison for the Components of the Proposed Aggregation Method in terms of Average Precision (\%).}
\label{table:compareconvandkernel}
\centering
\begin{tabular}{c|c c c c c}
\hline
\bfseries Method & \bfseries DTD & \bfseries FMD & \bfseries KTH-T2b & \bfseries CUB-200-2011 & \bfseries FGVC-aircraft \\
\hline
B-CNN \cite{Lin2016Bilinear}& $67.9$ & $77.8$ & $79.7$ & $74.0$ & $74.3$\\
\hline
$512$ conv $1\times1$ + B-CNN \cite{Lin2016Bilinear}& $66.3$ & $74.9$ & $77.9$ & $67.2$ & $65.3$\\
\hline
VGG-16 \cite{Simonyan2014Very} & $66.8$ & $77.8$ & $78.3$ & $68.0$ & $75.0$ \\
\hline
512 conv $1\times1$ + VGG-16 \cite{Simonyan2014Very} & $64.5$ & $74.3$ & $76.7$ & $64.5$ & $75.1$\\
\hline
\textbf{$no$ $conv$ $1\times1$ + Kernel Aggregation Layer} & $\mathbf{66.8}$ & $\mathbf{76.7}$ & $\mathbf{81.0}$ & $\mathbf{72.2}$ & $\mathbf{75.8}$\\
\hline
\textbf{$128$ $conv$ $1\times1$ + Kernel Aggregation Layer} & $\mathbf{67.8}$ & $\mathbf{78.3}$ & $\mathbf{81.1}$ & $\mathbf{64.6}$ & $\mathbf{72.3}$\\
\hline
\textbf{$512$ $conv$ $1\times1$ + Kernel Aggregation Layer} & $\mathbf{67.1}$ & $\mathbf{78.6}$ & $\mathbf{81.3}$ & $\mathbf{72.1}$ & $\mathbf{76.7}$\\
\hline
\textbf{Ours} & $\mathbf{68.9}$ & $\mathbf{79.2}$ & $\mathbf{81.1}$ & $\mathbf{72.4}$ & $\mathbf{77.8}$\\
\hline
\end{tabular}
\end{table*}

\begin{figure}[!t]
\centering
\includegraphics[width=3.5in]{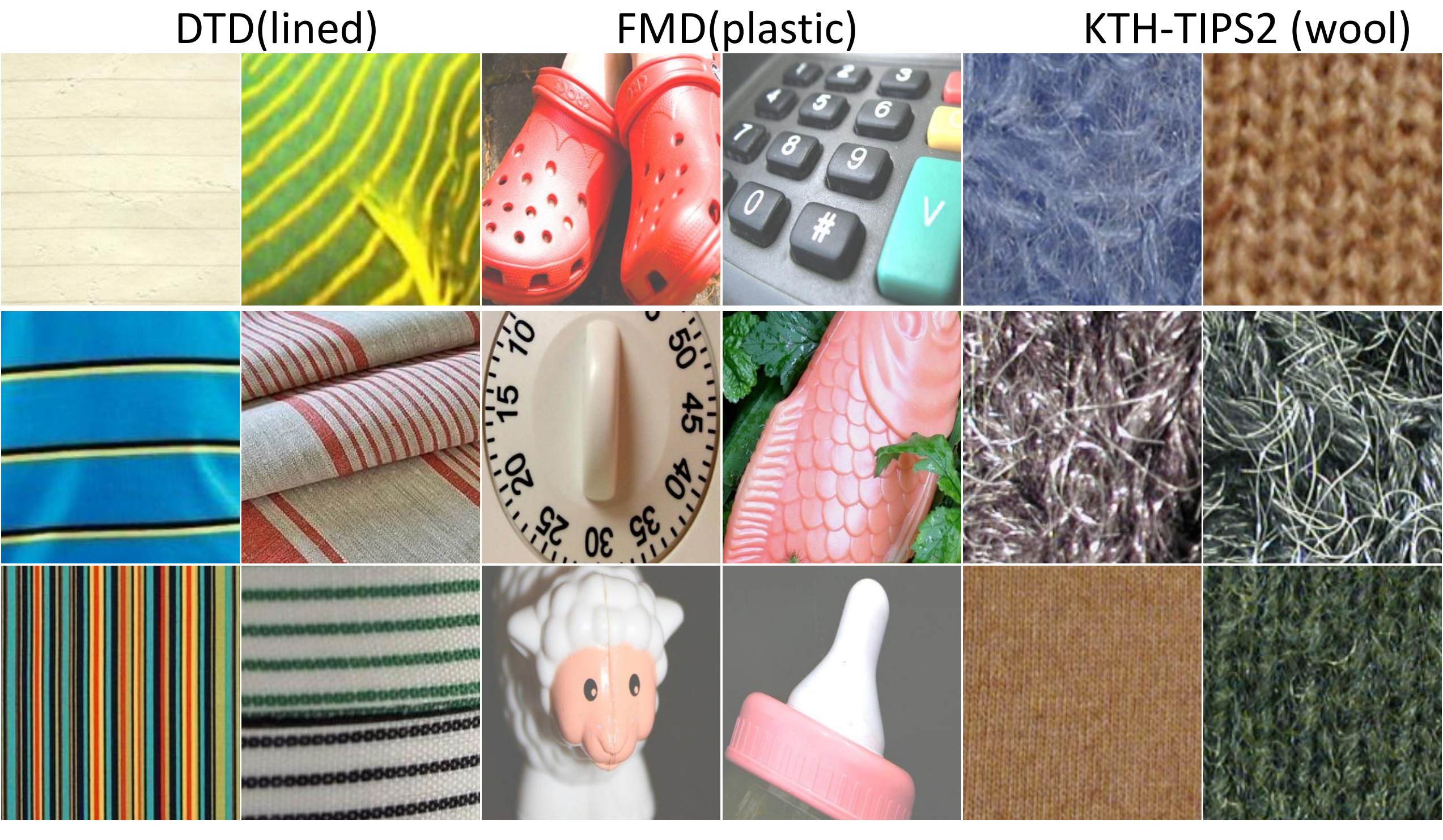}
\caption{Example images on the DTD, FMD and KTH-2b datasets. We can see that images of the same category have huge differences, particularly the plastic class.}
\label{fig:texturedatasets}
\end{figure}

\begin{figure}[!t]
\centering
\includegraphics[width=3.5in]{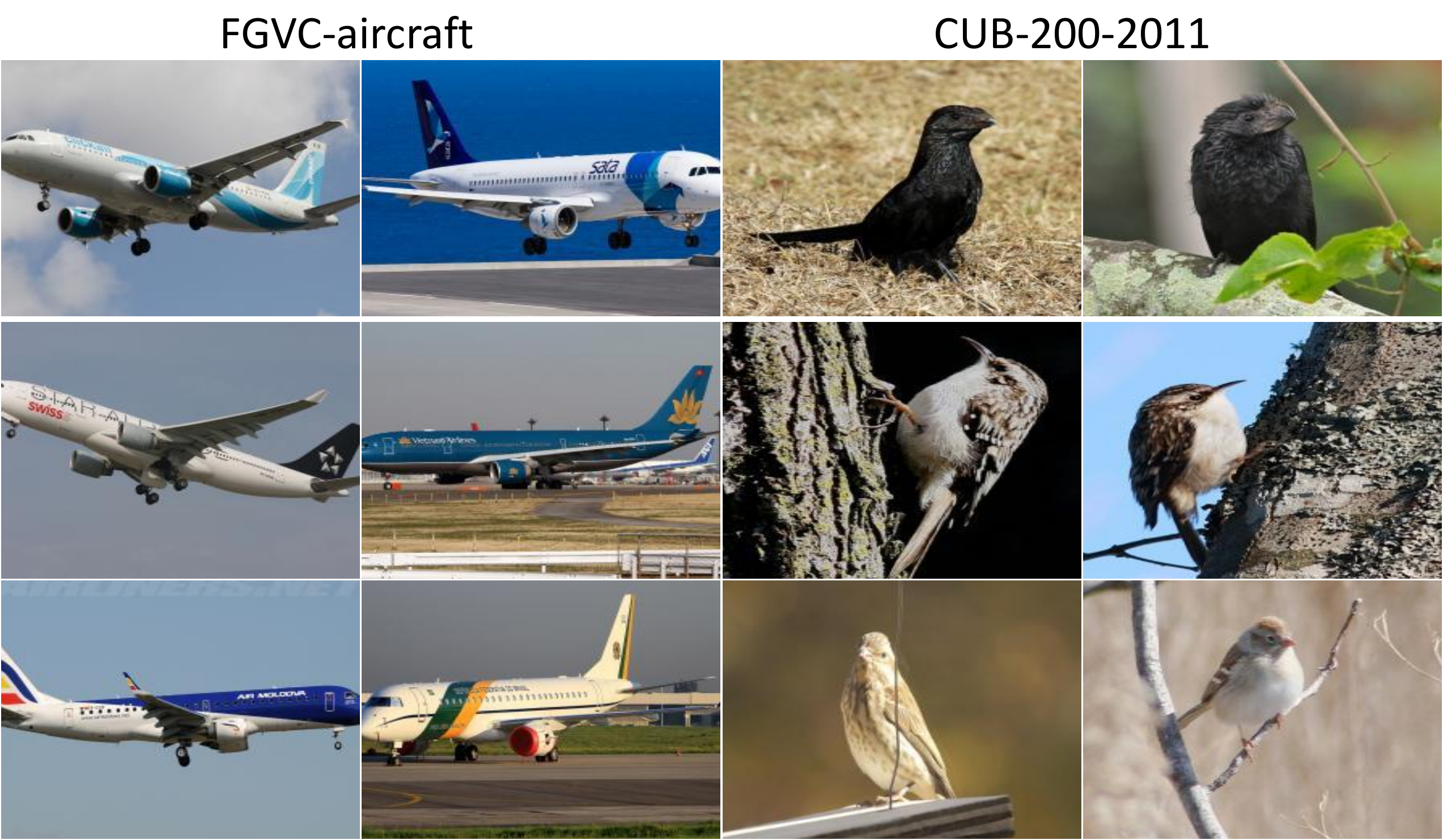}
\caption{Example reference images on the FGVC-aircraft and CUB-200-2011 datasets. The gaps between the pictures are very small.}
\label{fig:finegrainedtexture}
\end{figure}

\subsection{Datasets and Evaluation Protocols}

We choose three texture datasets in the experiments. They are Describable Textures Dataset (DTD) \cite{Cimpoi2014Describing}, Flickr Material Database (FMD) \cite{Sharan2013Recognizing} and KTH-TIPS-2b (KTH-2b) \cite{Caputo2005Class}. DTD and FMD are both collected in the wild conditions while KTH-2b is under the laboratory condition. DTD has $47$ classes, and each class contains $120$ images. There are totally $5640$ images in DTD. FMD contains $1000$ images of $10$ classes, each class has $100$ images. KTH-2b contains $4752$ images of $11$ classes. Fig.~\ref{fig:texturedatasets} illustrates the texture datasets for our experiments. For these texture datasets, we follow the standard train-test protocol. We divide DTD and FMD into three subsets randomly, and use two subsets for training and the rest one subset for testing. Images of KTH-2b are splited into four samples. we train the framework using one sample and test on the rest three samples. Inspired by \cite{Krizhevsky2012ImageNet}, the texture images are augmented. We do $15$ times augmentation to the training data, including randomly cropping $10$ times and picking from center and four corners. The test images are only picked from the center and four corners. The size of cropped images are resized to $224 \times 224$.

We report results on birds and aircrafts fine-grained recognition datasets. The birds dataset \cite{Wah2011The} is CUB-200-2011 which contains $200$ classes and $11788$ images totally. The FGVC-aircraft dataset \cite{Maji2013Fine} contains $10000$ aircraft images of $100$ classes. Fig.~\ref{fig:finegrainedtexture} illustrates some fine-grained images. We train and test the birds and aircrafts fine-grained datasets through the inherent training document. The data augmentation is not applied to the fine-grained images. We resize them to the size of $224 \times 224$. All the texture and fine-grained images are normalized by subtracting means for RGB channels.

\subsection{Implementation Details}
The basic convolutional layers and pooling layers before our SPD aggregation are from the VGG-16 model which is pre-trained on the ImageNet dataset. We remove layers after the \emph{conv5-3} layer of VGG-16 model. Then we insert our SPD aggregation method into the network following the \emph{conv5-3} layer. Finally, a FC layer and a softmax layer follow the vectorization layer where the output dimension of the FC layer is equal to the number of classes. All our networks run under the caffe framework. We use SGD with a mini-batch size of $32$. The training process is divided into two stages. At the first training stage, we fix the parameters before the SPD aggregation method and train the rest new layers. The learning rate is started from $0.1$ and reduced by $10$ when error plateaus. At the second training stage, we train the whole network. The learning rate is started from $0.001$ and divided by $10$ when error plateaus.

\subsection{Experiments for the SPD Aggregation Framework}
In this section, we compare the SPD aggregation framework with some state-of-the-art convolutional feature aggregation methods. First a $1\times1$ convolutional layer whose number of channels is $512$ follows the \emph{conv5-3} convolutional layers. Then our SPD aggregation method including a kernel aggregation layer, an SPD matrix transformation layer and a vectorization layer is inserted after the $1\times1$ convolutional layer. The output size of the SPD matrix transformation layer is $512 \times 512$. Considering these datasets are not big enough, we only use one SPD matrix transformation layer to avoid the overfitting. Table \ref{table:comparetextureandfinegrained} shows the comparison on texture datasets and fine-grained datasets respectively.

The following methods are evaluated in our experiments, FV-CNN \cite{Cimpoi2015Deep}, FV-FC-CNN \cite{Cimpoi2015Deep}, B-CNN \cite{Lin2016Bilinear}, Deep-TEN \cite{Zhang2016Deep} and the pure VGG-16 model \cite{Simonyan2014Very} is used as the baseline. FV-CNN aggregates convolutional features from VGG-16 \emph{conv5-3} layer. The dimension of it is $65600$ and is compressed to $4096$ by PCA for classification. FV-FC-CNN incorporates the FC features and FV vector. B-CNN uses the Bilinear pooling method on the \emph{conv5-3} layer of VGG-16 model. The Deep-TEN uses $50$-layers ResNet and larger number of training samples and image size, while the other methods use VGG-16 model. It is not scientific to compare Deep-TEN with the other feature aggregation methods.

We can see that, our method gets a better performance, especially on KTH-2b and FGVC-aircraft datasets. The average precision of our method on KTH-2b and FGVC-aircraft datasets are $81.1\%$ and $77.8\%$. In contrast, B-CNN achieves $79.7\%$ and $74.3\%$. On DTD and CUB-200-2011 datasets, our method is slightly worse than the B-CNN. The reason may be that the linear relationships among features are dominant on some datasests and the nonlinear relationships are important on the others.

\subsection{Experiments for the Components of the Proposed Aggregation Method} \

\textbf{\bm{$1\times1$} convolutional layer.} As mentioned above, we employ a $1\times1$ convolutional layer to accomplish the preprocessing of convolutional features. In this section, we provide experiments for the necessity of the preprocessing of convolutional features in our method. We design experiments in Table \ref{table:compareconvandkernel}. We combine different numbers of channels of $conv$ $1\times 1$ layer with the kernel aggregation layer. We also add the $conv$ $1\times 1$ layer to the B-CNN and pure VGG network. $No$ $conv$ $1 \times 1$, $128$ $conv$ $1\times 1$ and $512$ $conv$ $1\times 1$ in the Table indicate that whether there is a $conv$ $1\times1$ layer before the kernel aggregation layer and the number of channels of the $conv$ $1\times1$ layer. The kernel aggregation layer in the Table means that there is only the kernel aggregation layer without the SPD matrix transformation layer and the vectorization layer in the network. Table~\ref{table:compareconvandkernel} shows that, the $conv$ $1\times 1$ layer is beneficial to our nonlinear kernel aggregation method. However, it is useless or even harmful to the B-CNN and pure VGG-16 model. Our benefits are brought about by the powerful SPD matrix instead of the preprocessing of convolutional features. But the preprocessing of convolutional features can actually lead to better performance to the SPD aggregation. The reason may be that the convolutional features are totally different from the kernel matrices but have some similarities to the Bilinear matrices or FC features. We can also observe that the number of channels of $conv$ $1\times1$ layer has small influence for the texture datasets. But when it is reduced to $128$, the performance on fine-grained datasets is declined. So we argue that the convolutional features are redundant for the texture datasets but not redundant for the fine-grained datasets.

\textbf{SPD Matrix Generation Process.} To evaluate the effectiveness of the nonlinear SPD matrix generation process, we establish a network that only contains the kernel aggregation layer without the SPD matrix transformation layer and vectorization layer. The outcome is shown in Table~\ref{table:compareconvandkernel}. Without the $conv$ $1\times1$ layer, \emph{i.e.,} $no$ $conv$ $1\times1$ + kernel aggregation layer in the Table, our network is comparable with the B-CNN and pure VGG-16 model. When the $conv$ $1\times1$ layer is added to the network, \emph{i.e.,}, $512$ $conv$ $1\times1$ + Kernel Aggregation Layer in the Table, it has a obvious better performance than the other methods on FMD, KTH-2b and FGVC-aircraft datasets.

\textbf{SPD Matrix Transformation Process.} We design experiments to evaluate the effectiveness of the proposed transformation process in this subsection. Compared with ours, $512$ $conv$ $1\times1$ $+$ kernel aggregation layer in Table~\ref{table:compareconvandkernel} only lacks the SPD matrix transformation layer and the vectorization layer, the rest is the same. Through Table~\ref{table:compareconvandkernel}, we find that the SPD matrix transformation process transforms the SPD matrix to a more suitable and discriminative representation. Especially on DTD dataset and FGVC-aircraft datasets, the performance is improved by $1.8\%$ and $1.1\%$ respectively.

\section{Conclusion}
\label{conclusion}
In this paper, we have proposed a new powerful SPD aggregation method which models the convolutional feature aggregation as an SPD matrix non-linear learning problem on the Riemannain manifold. To achieve this goal, we have designed three new layers to aggregate the convolutional features into an SPD matrix and transform the SPD matrix to be more discriminative and suitable. The three layers include a kernel aggregation layer, an SPD matrix transformation layer and a vectorization layer under an end-to-end framework. We investigated the component decomposition and retraction of the Orthogonal Stiefle manifold to carry out the backpropagation of our model. Meanwhile, the faster matrix operation was adopted to speed up forward and backward backpropagations. Compared with alternative aggregation strategies such as FV, VLAD and bilinear pooling, our SPD aggregation achieves appealing performance on visual classification tasks. Extensive experiments on $11$ challenging datasets have demonstrated that our approach outperforms the state-of-the-art methods.


%

\ifCLASSOPTIONcaptionsoff
  \newpage
\fi




%

\bibliographystyle{IEEEtran}

\bibliography{IEEEabrv}




\end{document}